\documentclass{article}
\usepackage{log_2023}				

\usepackage{booktabs}						
\usepackage{multirow}						
\usepackage{amsfonts}						
\usepackage{graphicx}						
\usepackage{duckuments}						
\usepackage{quiver}
\usepackage{amssymb}
\usepackage{amsmath}
\usepackage{amsthm}
\usepackage{mathtools}
\usepackage{adjustbox}
\usepackage{paralist}

\makeatletter
\newcommand*\bigcdot{\mathpalette\bigcdot@{.5}}
\newcommand*\bigcdot@[2]{\mathbin{\vcenter{\hbox{\scalebox{#2}{$\m@th#1\bullet$}}}}}
\makeatother

\theoremstyle{plain}
\newtheorem{theorem}{Theorem}[section]
\newtheorem{proposition}[theorem]{Proposition}

\theoremstyle{definition}

\newtheorem{example}[theorem]{Example}
\theoremstyle{remark}

\renewcommand{\vec}{\mathbf}

\usetikzlibrary{automata,positioning,decorations.pathmorphing}

\definecolor{olivegreen}{rgb}{0,0.6,0}
\definecolor{mymauve}{HTML}{E60B42}
\definecolor{echopurple}{HTML}{9400D1}
\definecolor{camdrk}{HTML}{0099CC}
\definecolor{darkgry}{HTML}{333333}
\definecolor{echoblue}{HTML}{0099CC}
\definecolor{echonavy}{HTML}{0054B2}

\tikzstyle{vertex}=[circle,fill=darkgry!55,minimum size=20pt,inner sep=0pt, fill opacity=0.5]
\tikzstyle{selected vertex} = [vertex, fill=echopurple!15, fill opacity=0.5, text opacity=1.0, draw=mymauve, ultra thick]
\tikzstyle{select vertex} = [vertex, fill=echopurple!15, fill opacity=0.5, text opacity=1.0, draw=echonavy, ultra thick]
\tikzstyle{selectx vertex} = [vertex, fill=echopurple!24, fill opacity=1.0]

\tikzstyle{selectz vertex} = [vertex, fill=echonavy!24, fill opacity=1.0]
\tikzstyle{selectw vertex} = [vertex, fill=mymauve!24, fill opacity=1.0]
\tikzstyle{selectq vertex} = [vertex, fill=echopurple!24, fill opacity=1.0]
\tikzstyle{selectg vertex} = [vertex, fill=olivegreen!24, fill opacity=1.0]

\tikzstyle{edge} = [draw,thick,-]
\tikzstyle{full edge} = [draw,line width=5pt,-,red!50]
\tikzstyle{selected edge} = [draw,line width=5pt,-,echonavy!30]

\usepackage[numbers,compress,sort]{natbib}	

\title[Asynchronous Algorithmic Alignment with Cocycles]{Asynchronous Algorithmic Alignment with Cocycles}

\author[A. Dudzik et al.]{%
Andrew Dudzik\\
Google DeepMind\\
\email{adudzik@google.com}\And
Tamara von Glehn\\
Google DeepMind\\
\email{tamaravg@google.com}\And
Razvan Pascanu\\
Google DeepMind\\
\email{razp@google.com}\And
Petar Veli\v{c}kovi\'{c}\\
Google DeepMind\\
\email{petarv@google.com}
}

\begin{document}

\maketitle

\begin{abstract}
State-of-the-art neural algorithmic reasoners make use of message passing in graph neural networks (GNNs). But typical GNNs blur the distinction between the definition and invocation of the message function, forcing a node to send messages to its neighbours at every layer, synchronously. When applying GNNs to learn to execute dynamic programming algorithms, however, on most steps only a handful of the nodes would have meaningful updates to send. 
One, hence, runs the risk of inefficiencies by sending too much irrelevant data across the graph. But more importantly, many intermediate GNN steps have to 
learn the identity functions, which is a non-trivial learning problem.
%
In this work, we explicitly separate the concepts of node state update and message function invocation. With this separation, we obtain a mathematical formulation that allows us to reason about asynchronous computation in both algorithms and neural networks. Our analysis yields several practical implementations of synchronous scalable GNN layers that are provably invariant under various forms of asynchrony.
\end{abstract}

\section{Introduction}


The message passing primitive---performing computation by aggregating information sent between neighbouring entities \citep{gilmer2017neural,BattagliaPLRK16}---is known to be remarkably powerful. 
%
Message passing is the core primitive in \emph{graph neural networks} \citep[GNNs]{velivckovic2023everything}, a prominent family of deep learning models.
Owing to the ubiquity of graphs as an abstraction for describing the \emph{structure} of systems, GNNs have enjoyed immense popularity across both scientific \citep{wang2023scientific} and industrial applications, including novel drug screening \citep{stokes2020deep,Liu2023}, designing next-generation machine learning chips  \citep{mirhoseini2021graph}, 
serving travel-time estimates \citep{derrow2021eta},
 particle physics~\citep{dezoort2023graph}, and settling long-standing problems in pure mathematics~\citep{davies2021advancing,blundell2021towards,williamson2023deep}.

Another active area of research for GNNs is \emph{neural algorithmic reasoning} \citep[NAR]{velivckovic2021neural}. NAR seeks to design neural network architectures that capture \emph{classical computation}, largely by learning to execute it \citep{velivckovic2019neural}. This is an important problem in the light of today's large-scale models, as they tend to struggle in performing exactly the kinds of computations that classical algorithms can trivially capture \citep{lewkowycz2022solving}. 

The use of GNNs in NAR is largely due to \emph{algorithmic alignment} \citep{xu2019can}: the observation that, as we increase the structural similarity between a neural network and an algorithm, it will be able to learn to execute this algorithm with improved sample complexity. 
%
Here, we make novel contributions to the theory of algorithmic alignment. Our approach ``zooms in'' on the theoretical analysis in \citep{dudzik2022graph}, which analysed computations---of both algorithms and GNNs---globally. Instead, we center our discussion on a \emph{node-centric}\footnote{As we will expand on later in the work, \emph{node} is a misnomer for what we precisely mean, but to improve the exposition we will rely on this term for now.} view: analysing computations around individual nodes in the graph, in isolation. 

\begin{figure}
    \centering
    \includegraphics[width=\linewidth]{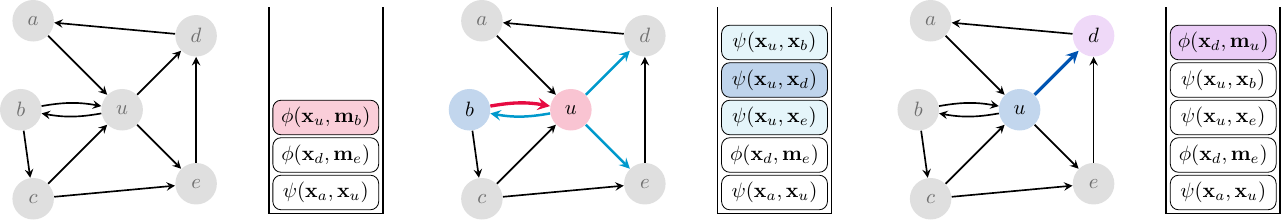}
    \caption{A possible execution trace of an asynchronous GNN. While traditional GNNs send and receive all messages synchronously, under our framework, at any step the GNN may choose to execute any number of possible operations (depicted here with a collection on the right side of the graph). Note that we do \emph{not} aim to \emph{implement} an asynchronous GNN---a feat concurrently explored by AMP \citep{faber2022asynchronous} and Co-GNN \citep{finkelshtein2023cooperative}---rather, we seek to build \emph{synchronous} GNNs that are \emph{invariant}---i.e., will yield identical output node embeddings---under various forms of asynchronous execution.} 
    \label{fig:async_gnn}
\end{figure}

This view allows us to study message passing under various synchronisation regimes and can help us identify choices of message functions that better align with target algorithms, in a manner that was not possible under previous frameworks---indeed, it allows us to theoretically justify the unreasonable effectiveness of architectures such as PathGNNs \citep{tang2020towards}. We refer to our new framework as \emph{asynchronous algorithmic alignment}, and formalise it using tools of category theory, monoid actions, and cocycles. To visualise what executing an asynchronous GNN might look like, refer to Figure \ref{fig:async_gnn}.  

\section{Message passing and (asynchronous) algorithms: Addition with carry}
To better understand the concept of algorithmic alignment, and in particular what we mean by a \emph{node-centric} perspective, we will utilise a simple illustrative example---namely \emph{addition with carry}---and present it in the framework of message passing that GNNs rely on.

{\bf Message passing framework.} We will use the definition of GNNs based on \citet{bronstein2021geometric}. Let a graph be a tuple of \emph{nodes} and \emph{edges}, $G=(V, E)$, with one-hop neighbourhoods defined as $\mathcal{N}_u = \{v\in{V}\ |\ (v, u)\in{E}\}$. Further, a node feature matrix ${\bf X}\in\mathbb{R}^{|{V}|\times k}$ gives the features of node $u$ as $\vec{x}_u$; we omit edge- and graph-level features for clarity. A \emph{(message passing)} GNN over this graph is then executed as:
\begin{equation}\label{eqn:gnn}
    \vec{x}_u' = \phi\left(\vec{x}_u, \bigoplus\limits_{v\in\mathcal{N}_u} \psi(\vec{x}_u, \vec{x}_v)\right)
\end{equation}
To put this equation in more abstract terms, we review the diagram of the message-passing framework of  \citet{dudzik2022graph}, with the addition of the message function $\psi$ to emphasise symmetry:
\[\begin{tikzcd}
	{\mathtt{args}} && {\mathtt{edge}} && {\mathtt{msgs}} \\
	\\
	\\
	\\
	{\mathtt{senders}} &&&& {\mathtt{receivers}}
	\arrow["\bigotimes"{description}, from=1-1, to=1-3]
	\arrow["\bigoplus"{description}, from=1-5, to=5-5]
	\arrow["{\mathsf{scatter}}"{description}, from=1-3, to=5-5]
	\arrow["{\mathsf{gather}}"{description}, from=5-1, to=1-3]
	\arrow["{\mathsf{copy}}"{description}, from=5-1, to=1-1]
	\arrow["{\psi}"{description}, from=1-3, to=1-5]
\end{tikzcd}\]


First, sender features ($\mathtt{senders}$) are duplicated along outgoing edges to form the arguments ($\mathtt{args}$) to a message function $\psi$. These arguments are collected into a list using the $\bigotimes$ operator---which is traditionally a concatenation, though as per \citet{dudzik2022graph}, it can be any operator with a monoidal structure. This list of arguments now lives on a new, transient, $\mathtt{edge}$ datatype. These two steps constitute a \emph{gather} operation. In Equation \ref{eqn:gnn}, this corresponds to copying the node features in $\vec{X}$, considered as a $V$-indexed tensor, to obtain feature pairs $(\vec{x}_u, \vec{x}_v)$, as an $E$-indexed tensor.

Note that the $\mathtt{senders}$ do not necessarily correspond to the more established notion of ``sender nodes'' in graph representation learning \citep{battaglia2018relational}---rather, we consider a node to be a sender if its features are \emph{necessary} to compute the message function $\psi$. Hence, in Equation \ref{eqn:gnn}, we consider both the features of nodes $u$ and $v$ to be ``senders'', rather than assuming that only $v$ is a sender. This allows us to easily extend this idea to messages spanning arbitrary numbers of inputs living on various places in the graph. For example, if edge features are used for $\psi$, we can include them as part of $\mathtt{senders}$ also.

Next, we perform a similar operation, a \emph{scatter}, by first applying the message function $\psi$ to the arguments, which computes the messages to be sent ($\mathtt{msgs}$). Then, messages are copied to suitable receivers, which aggregate along their incoming edges to form the final set of receiver node features ($\mathtt{receivers}$). In Equation \ref{eqn:gnn}, this refers to the application of the message function $\psi$, and the aggregation $\bigoplus$. This general gather-scatter paradigm can be seen throughout the literature on message passing, for example in the sheaf Laplacian of \citet{bodnar2023neural}.

Adding graph- and edge-level features complicates our exposition, also potentially creating some confusion within the diagrams above. In this work we will argue that, in order to simplify our theoretical treatment of the computations carried by GNNs and algorithms, \emph{it is useful to lift any graph component that has a \textbf{persistent} state---e.g. a feature vector updated at every layer---to the status of \textbf{node}, while \textbf{edges} are \textbf{transient} features created as part of the model's information flow}. For any graph, is trivial to construct a corresponding graph that has the above property. From now on, we assume that anything with a persistent state is a node in the graph, while edges are always transient.

{\bf Asynchronicity and constructing arguments.} Now we discuss the \emph{addition with carry} algorithm on a high level---we will focus on its details in Section \ref{sec:addcarr}. Intuitively, digits would represent nodes with persistent state, while edges show how computation needs to be carried, i.e. which digits need to be added. Because of the carry, we quickly notice that there is a misalignment between the algorithm and its current graph representation. On one hand, the algorithm requires a certain level of asynchrony, where digits need to be added in sequence, in order to be able to correctly accumulate the carry. The parallel computation imposed by having all nodes sending messages within their neighbourhoods, synchronously, can be problematic, requiring many nodes to learn the identity function on most steps. 
Fitting such sparse-update computation with (G)NNs has led to solutions that were either very brittle \citep{graves2014neural,graves2016hybrid} or requiring vast amounts of strong supervision \citep{velivckovic2020pointer,strathmann2021persistent}. This problem gets even further exacerbated when GNNs are used to fit multiple algorithms at once \citep{xhonneux2021transfer,ibarz2022generalist}. 


A second distinct issue in terms of executing the algorithm is: how will the carry be handled? Indeed, assuming the carry can not be added to the state of the node (which is meant to be only a single digit), the model is unable to properly represent the required computation.

To resolve this issue we can focus on how messages and the persistent state of the node get transformed into new states and \emph{arguments} for the message function $\psi$. In particular we augment the computation of the node in order to capture not only how the new state is constructed but also how the node constructs \emph{new arguments} for the next computational step, allowing us to chain these computations.  We capture the stateful nature of this computation in the following complementary diagram:

\adjustbox{scale=0.95,center}{
\begin{tikzcd}
	&& {\mathtt{edge}} \\
	\\
	\\
	\\
	{\mathtt{senders}} && {\mathtt{persistent}} \arrow[out=60,in=120,loop,"\phi{,}\delta",swap] && {\mathtt{receivers}}
	\arrow["{\mathsf{gather}}"{description}, from=5-1, to=1-3]
	\arrow["{\mathsf{scatter}}"{description}, from=1-3, to=5-5]
	\arrow[from=5-5, to=5-3]
	\arrow[from=5-3, to=5-1]
\end{tikzcd}
}



We will make one key assumption: the sender and receiver features should each have the structure of a \emph{monoid}---we have included an introduction to the theory of monoids in Appendix \ref{app:monoids}. This implies that we can think of both arguments and messages as being sent in chunks, which are assembled in order. We can think of these monoids as \emph{instruction queues}, where the monoid operation corresponds to instruction composition. Our main definitions will not assume commutativity, though it holds in many examples---we may reorder instructions arbitrarily if so.

\section{Node-centric view on algorithmic alignment}

For this section, we focus our attention on a single node, and explore the relationship between the message monoid, which we write as $(M,\cdot, 1)$ using multiplicative notation, and the argument monoid, which we write as $(A,+, 0)$ using additive notation.

Suppose that the internal state takes values in a set $S$. The process by which a received message updates the state and produces an argument is described by a function $M\times S \to S\times A$. This is equivalent, by currying, to a Kleisli arrow $M\to [S,S\times A] = \mathtt{state}_S(A)$ for the state monad.

Given a pair $(m,s)$, we denote the image under this function by $(m\bigcdot s, \delta_m (s))$, where $\bigcdot:M\times S \to S$ is written as a binary operation and $\delta:M\times S \to A$ is described by some argument function $\delta$. First, we look into the properties of $\bigcdot$. 

Each incoming message (an element of $M$) transforms the state (an element of $S$) in some way. We assume that the multiplication of $M$ corresponds to a composition of these transformations. Specifically, we ask that $\bigcdot$ satisfies the unit and associativity axioms:
\begin{equation}\label{action}
\begin{split}
1\bigcdot s &= s \\
(n\cdot m)\bigcdot s &= n\bigcdot (m\bigcdot s)
\end{split}
\end{equation}
Next, we interpret Equation \ref{action} in terms of the argument function $\delta$. In the first equation, the action $1\bigcdot s$ generates an argument $\delta_1(s)$. But on the right-hand side there is no action, so no argument is produced. In order to process both sides consistently, $\delta_1(s)$ must be the zero argument.

Similarly, in the second equation, the left-hand side produces one argument $\delta_{n\cdot m}(s)$, while the right-hand side produces two, $\delta_m(s)$ and then $\delta_n(m\bigcdot s)$. Setting these equal, we have the following:
\begin{equation}\label{argumentequation}
\begin{split}
\delta_1(s) &= 0 \\
\delta_{n\cdot m}(s) &=  \delta_n(m\bigcdot s) + \delta_m(s)
\end{split}
\end{equation}
Equivalently, we could arrive at these equations by extending the action of $M$ on $S$ to one of $M$ on $S\times A$, as follows. Given a pair $(s,a)$ of a state $s$ together with an outgoing argument $a$, we act on the state, while generating a new argument that gets added to the old one: $m\star (s,a) := (m\bigcdot s, \delta_m(s)+a)$. One can show that Equation \ref{argumentequation} is exactly the unit and associativity axioms for the operator~$\star$.

{\bf Cocycles.} Equation \ref{argumentequation} can be rewritten in a more elegant form, but this requires a few preliminaries.

First, consider the set $F = [S, A]$ of \emph{readout functions}; functions mapping states to corresponding arguments. $F$ inherits structure from both $S$ and $A$. First, $F$ is a monoid because $A$ is; we define a zero function $0(s):= 0$ and function addition $(f+g)(s) := f(s) + g(s)$.

Second, $F$ has an action of $M$, given by $(f\bigcdot m)(s):=f(m\bigcdot s)$. This is a \emph{right} action rather than a left action; the associativity axiom $f\bigcdot (m\cdot n) = (f\bigcdot m)\bigcdot n$ holds, but the reversed axiom may not.

With these definitions, if we write $\delta$ in its curried form $D: M\to [S, A]$, we can rewrite Equation \ref{argumentequation}:
\begin{equation}\label{cocycle}
\begin{split}
D(1) &= 0 \\
D(n\cdot m) &= D(n)\bigcdot m + D(m)
\end{split}
\end{equation}
Equation \ref{cocycle} specifies that $D$ is a \emph{1-cocycle}, otherwise known as a \emph{derivation}. To understand the latter term, consider the more general situation where $F$ also has a left action of $M$. Then we may write $D(n\cdot m) =  D(n)\bigcdot m + n\bigcdot D(m)$, which is just the Leibniz rule for the derivative. Our equation describes the special case where this left action is trivial.

We summarise the above chain of deductions as follows:

\begin{proposition}
A node equipped with a rule $D$ for generating arguments is invariant to asynchrony, i.e. its output does not depend on the grouping of incoming messages, if and only if $D$ is a 1-cocycle.
\end{proposition}


{\bf Edges: homomorphisms and multimorphisms.}
\label{multimorphisms} 1-cocycles also appear in the literature under the name \emph{crossed homomorphisms}. Indeed, if the right action of $M$ on $[S,A]$ is trivial, the above equations are $D(1)=0$ and $D(n\cdot m) =  D(n) + D(m)$, i.e. $D$ is a homomorphism of monoids.

We can use this observation to describe edges---at least, edges with only one input---as stateless nodes. We simply reverse the roles of argument and message monoids: for asynchronous communication over edges, we require that the message function $\psi$ satisfies $\psi(0) = 1$ (null arguments generate null messages) and $\psi(a+b)=\psi(a)\cdot\psi(b)$ (aggregating before $\psi$ gives same results as aggregating after). In other words:

\begin{proposition}
An edge with a single-input message function $\psi$ supports asynchronous invocation if and only if $\psi$ is a homomorphism of monoids.
\end{proposition}

How to extend to edges that take $k$ inputs, for example, edge features or receiver features? \citet{dudzik2022graph} proposed repeated multiplication in a semiring, but for asynchrony, we only require the much weaker requirement that $\psi$ be a \emph{multimorphism}, i.e. given any fixed values for $k-1$ of the variables, $\psi$ is a homomorphism in the remaining variables.\footnote{This is the analogue of multilinearity for maps of vector spaces.}

With this in mind, if $M_1, \cdots , M_k$ are commutative monoids, we simply treat a stateful node with multiple message inputs $M_1, \cdots , M_k$ as a stateful node with a single message input given by the tensor product of commutative monoids $M_1 \otimes \cdots , \otimes M_k$\footnote{Commutativity is needed here, otherwise the tensor product may not exist.}. In the stateless case, this is equivalent, by the universal property of the tensor product, to the above condition that $\psi$ is a multimorphism, so this is a convenient mathematical way to describe asynchrony with respect to multiple arguments.


{\bf Example: The natural numbers and addition with carry.}\label{sec:addcarr} If $M=(\mathbb{N},+,0)$ is the monoid of natural numbers under addition,\footnote{We are using additive notation here, so the identity element is denoted with $0$ instead of $1$.} then an action of $M$ on a set $S$ is equivalently an endofunction $\pi: S\to S$, with $m\cdot s := \pi^m (s)$. We now characterise all possible cocycles $M \to [S,A]$.

\begin{proposition}
Given any function $\omega: S\to A$, there is a unique $1$-cocycle $\delta$ with $\delta_1 = \omega$.
\end{proposition}
\begin{proof}
If $\delta$ is a cocycle and $\omega = \delta_1$, the cocycle condition gives us an inductive definition of $\delta$: $\delta_{n+1}(s) = \delta_n(s) + \delta_1 (s+n)$. On the other hand, given $\omega$ we can define: $\delta_n(s) := \sum_{i=0}^{n-1} \omega(s+i)$.

We note that $\delta_{m+n} = \sum_{i=0}^{m-1} \omega(s+i) + \sum_{j=m}^{m+n-1} \omega(s+j) =\sum_{i=0}^{m-1} \omega(s+i) + \sum_{j=0}^{n-1} \omega(m+s+j) = \delta_m(s) + \delta_n(m+s)$ so $\delta$ is a $1$-cocycle.
\end{proof}
Now, we consider the example of digit arithmetic with carry. Suppose that $S = \{0,\cdots, 9\}$ is the set of digits in base $10$, and our action is given by the permutation $\pi(s) = \overline{s+1}$, where the bar denotes reduction modulo $10$. If $M = (\mathbb{N}, +)$, we can define an argument function based on the number of carries produced when $1$ is added to $s$, $m$ times: $\delta_m(s) := \left\lfloor\frac{m+s}{10}\right\rfloor$.

\begin{proposition}
$\delta$ is a $1$-cocycle $M\to [S,A]$.
\end{proposition}
\begin{proof}
We directly verify that, for all $m,n\in\mathbb{N}$, $s\in \{0,\ldots, 9\}$: $\left\lfloor\frac{m+n+s}{10}\right\rfloor = \left\lfloor\frac{m+\overline{n+s}}{10}\right\rfloor + \left\lfloor\frac{n+s}{10}\right\rfloor$. This follows quickly by induction on $m$: for $m=0$ the two sides are equal, and as $m$ increments by $1$, the LHS and RHS are both incremented if and only if $m+n+s+1\equiv 0\pmod {10}$.
\end{proof}

\section{(A)synchrony and idempotence in GNNs}

Having discussed the general mathematical situation, we can now leverage the cocycle conditions to more rigorously discuss the synchronisation of GNN operations (such as gathers, scatters, and applications of $\psi$ or $\phi$). In the process of our discussion, we will show how our theory elegantly re-derives and extends a well-known neural algorithmic reasoning model.

In the study of algorithms, we are especially interested in a certain property of monoids: \emph{idempotence}. We start by proving a highly useful fact about the relationship between idempotence and cocycles.

\subsection{Idempotent monoids}

We say that $A$ is \emph{idempotent} if $a+a=a$ for all $a\in A.$

\begin{proposition}\label{idempotence}
Suppose that $S=A$. Define:

\[\delta_m(s) := \begin{cases*}
0 & if $m=1$ \\
m \bigcdot s & otherwise
\end{cases*}\]

If $\delta$ is a $1$-cocycle, then $A$ is idempotent. If $M=S=A$ and $\cdot = +$, the converse holds.
\end{proposition}
\begin{proof}
In this case, the second equation of Equation \ref{argumentequation} becomes $(n\cdot m)\bigcdot s=n\bigcdot(m\bigcdot s)+m\cdot s$. Setting $n=m=1$ gives $s=s+s$.

If $\cdot = +$ and $A$ is idempotent, then the equation is $n+m+s=(n+m+s)+(m+s)$, which holds since $(n+m+s)+(m+s)=n+((m+s)+(m+s))=n+m+s$.
\end{proof}

\subsection{Making $\phi$ a cocycle enables asynchrony}

Now we can explicitly formalise the residual map $\phi$ in Equation \ref{eqn:gnn}: it is just another description of what we have called an ``action''. That is, we have $\phi(s, m) = m\bigcdot s$ for all node features $s$ and (aggregated) non-null messages $m$.

{\bf Message aggregation asynchrony.}
$\bigoplus$ is usually taken, axiomatically, to be the operation of a commutative monoid \citep{ong2022learnable}. This is to emphasise the importance of message aggregation that does not depend on the order in which the messages are received. That is, letting $\bigoplus$ define a commutative monoid operation already allows us to support a certain form of asynchrony: we can aggregate messages online as we receive them, rather than waiting for all of them before triggering $\bigoplus$.

{\bf Node update asynchrony.}
Similarly, the axiom that $\phi$ defines an associative operation, as in Equation \ref{action}, corresponds to another type of asynchrony. When $\phi$ satisfies the associativity equation:
\begin{equation}
\phi(s, m\bigoplus n) = \phi(\phi(s, m), n)
\end{equation}
this tells us that $\phi$ itself can be applied asynchronously. Put differently, after each message arrives into the receiver node, we can use it to update the node features by triggering $\phi$, without waiting for the messages to be fully aggregated.

One common way to enforce associativity is to take $\phi = \bigoplus$, in which case the associativity of $\phi$ follows from the assumed associativity of $\bigoplus$.

{\bf Argument generation asynchrony.}
While these two conditions allow us to reason about the asynchrony in how incoming messages are processed, and the asynchrony in how the node's features are updated, what does the cocycle condition (Equation \ref{argumentequation}) mean for a GNN? 

Recall, the cocycle condition concerns the argument function $\delta$, which determines which arguments are produced after a node update. Specifically, the cocycle condition allows us to express the arguments produced by receiving two messages together ($\delta_{n\cdot m}$) as a combination of the arguments produced by receiving them in isolation ($\delta_n$ and $\delta_m$). Therefore, it gives us a mechanism that allows for each sender node to prepare their arguments to the message function, $\psi$, asynchronously, rather than waiting for all the relevant node updates to complete first.

Note that Equation \ref{eqn:gnn} does not distinguish between node features and sent arguments. In other words, it implicitly defines the argument function $\delta_m(s) = m\bigcdot s = \phi(s, m)$. Knowing this, we can leverage the converse direction of Proposition \ref{idempotence}, to provide some conditions under which the update function $\phi$ will satisfy the cocycle condition:

First, we require the state update (which we previously set to $\phi=\bigoplus$) to be idempotent. Not all commutative monoids are idempotent; $\bigoplus=\max$ is idempotent, while aggregators like $\operatorname{sum}$ are not. Note that this aligns with the known utility of the $\max$ aggregation in algorithmic tasks \citep{velivckovic2019neural,xu2020neural,dudzik2022graph}.

Second, we require $M=S=A$, which means that the messages sent must come from the same set of values as the node features and the arguments to the message function. This means that the dimensionality of the node features, arguments and messages should be the same---or alternately, that invoking the message or update functions should not change this dimensionality. This can be related to the encode-process-decode paradigm \citep{hamrick2018relational}: advocating for the use of a \emph{processor module}, repeatedly applied to its inputs for a certain number of steps. 

Note that this is only one possible way to enforce the cocycle condition in $\phi$; we remark that there might be more approaches to achieving this, including approximating the cocycle condition by optimising relevant loss functions.


\subsection{The rich asynchrony of $\max$-$\max$ GNNs: Rediscovering PathGNNs}

It is worthwhile to take a brief pause and collect all of the constraints we have gathered so far:
\begin{inparaenum}
    \item[(1)] $\bigoplus$ is a commutative monoid;
    \item[(2)] $\phi = \bigoplus$;
    \item[(3)] $\bigoplus$ must be idempotent, e.g., $\bigoplus=\max$;
    \item[(4)] The message function should output messages of the same dimensionality as node features. 
\end{inparaenum}

We can observe that Equation \ref{eqn:gnn} now looks as follows:
\begin{equation}\label{eqn:pathgnn}
    \vec{x}_u' = \max\left(\vec{x}_u, \max_{v\in\mathcal{N}_u} \psi(\vec{x}_u, \vec{x}_v)\right)
\end{equation}
Such a $\max$-$\max$ GNN variant allows for a high level of \emph{asynchrony}: messages can be sent, received and processed in an arbitrary order, arguments can be prepared in an online fashion, and it is mathematically guaranteed that the outcome will be identical as if we fully synchronise all of these steps, as is the case in typical GNN implementations. We also remark that Equation \ref{eqn:pathgnn} is almost exactly equal to the hard variant of the PathGNN model from \citet{tang2020towards}---the only missing aspect is to remove the dependence of $\psi$ on the receiver node (i.e., to remove $u$ from the $\mathtt{senders}$).

The only remaining point of synchronisation is the invocation of the message function, $\psi$; messages can only be generated once all of the arguments for the message function are fully prepared (i.e., no invocations of $\phi$ are left to perform for the relevant sender nodes).

\subsection{Extending PathGNNs with multimorphisms: Message generation asynchrony}



PathGNNs give one example of \emph{isotropic} message passing, where the message function for each edge, $\psi_e$, is a function of a single variable, the sender argument, and produces a single output, a receiver message. That is, we have a function $\psi_e:(A_{s(e)},+,0)\to(M_{t(e)}, \cdot, 1)$.


As we elaborated in Section \ref{multimorphisms}, the condition needed for argument asynchrony and message asynchrony to be compatible is that $\psi$ is a monoid homomorphism. This means that $\psi$ can be called---and messages generated---even before its arguments are fully ready, so long as it is called again each time the arguments are updated.

Note that PathGNNs, in their default formulation, do not always satisfy this constraint. We have, therefore, used our analysis to find a way to extend PathGNNs to a \emph{fully asynchronous} model. One way to obtain such a GNN---assuming it is already isotropic---is to make $\psi$ be a \emph{tropical linear} transformation. That is, $\psi$ would be parametrised by a $k\times k$ matrix, which is multiplied with $\vec{x}_v$, but replacing ``$+$'' with $\max$ and ``$\times$'' with $+$.

In the non-isotropic case \citep{bresson2017residual}, where $\psi$ may take multiple arguments, different properties of $\psi$ may correspond to different forms of asynchrony. Since various DP algorithms can be parallelised in many different ways, we consider this case a promising avenue for future exploration. In Section \ref{sec:eval} we give empirical results for multimorphisms as asynchronous message functions in anisotropic GNNs.

\subsection{Example: Semilattices and Bellman-Ford}

Originally, PathGNN was designed to align with the Bellman-Ford algorithm \citep{bellman1958routing}, due to its claimed structural similarity to the algorithm's operation---though this claim was not rigorously established. Now, we can rigorously conclude where this alignment comes from: the choice of aggregator ($\max$) and making $\psi$ only dependent on one sender node is fully aligned with the Bellman-Ford algorithm (as in \citet{xu2020neural}), \emph{and} both PathGNNs and Bellman-Ford can be implemented asynchronously without any errors in the final output (a novel conclusion enabled by our mathematical framework). We have already showed that PathGNNs satisfy the cocycle condition; now we will prove the same statement about Bellman-Ford, in order to complete our argument.




Let $P$ be any join-semilattice, i.e. a poset with all finite joins, including the empty join, which we denote $0$ as it is a lower bound for $P$. A standard result says that $P$ is equivalently a commutative, idempotent monoid $(P, \vee, 0)$. When $P$ is totally ordered, we usually write $\vee = \sup$ or $\max$.

In Bellman-Ford, $P$ will generally be either the set of all path lengths, (taken with a negative sign) or the set of all relevant paths ending at the current vertex, equipped with a total order that disambiguates between paths of equal length.

We set $M=S=A=P$, with action of $M$ on $S$ given by $m\cdot s := m\vee s$.

Since we want to inform our neighbors when our state improves, i.e. gets smaller, we could likewise define our argument function $\delta$ by $\delta_m (s) := m\vee s$, as this in fact satisfies the cocycle equation due to Proposition \ref{idempotence}. However, this will lead to an algorithm that never terminates, as new redundant arguments will continue to be generated at each step.

So we define the argument function a bit more delicately:
\[\delta_m(s) := \begin{cases*}
0 & if $m\leq s$ \\
m \vee s & otherwise
\end{cases*}\]

\begin{proposition}
$\delta$ is a $1$-cocycle $M\to [S,A]$.
\end{proposition}
\begin{proof}
Pick $m,n \in M$, $s\in S$. We wish to show that $\delta_{m\vee n}(s) = \delta_m(n\vee s) \vee \delta_n(s)$.

If $m,n\leq s$ then both sides equal $0$. If $m\leq s$ but $n\not\leq s$ then the LHS is $m\vee n\vee s=n\vee s$ while the RHS is $0\vee (n\vee s)$. If $n\leq s$ but $m\not\leq s$ then the LHS is $m\vee n \vee s = m\vee s$ while the RHS is $(m\vee n \vee s)\vee 0 = m\vee s$.
\end{proof}
Now, we can see that the definition of $\delta$ prevents new arguments from being generated once the optimal value has been obtained, if we take the convention that zero arguments are ignored. In particular, if $P$ is well-ordered, only finitely many arguments can be generated regardless of the input.

Lastly, note that Bellman-Ford is also a perfect example of message generation asynchrony; its $\psi_e$, for each edge, simply adds the edge length of $e$ to the sender node distance. $+$ is easily seen to distribute over $\max$, making $\psi_e$ a monoid homomorphism. All conditions combined, Bellman-Ford can be indeed made fully asynchronous, without sacrificing the fidelity of the final output. This solidifies the algorithmic alignment of our PathGNN variant with Bellman-Ford.

\section{Evaluating asynchrony-invariant GNNs}\label{sec:eval}

In our paper, we introduced three\footnote{Strictly speaking, our analysis predicts \emph{four} levels, but as in typical GNNs (of Equation \ref{eqn:gnn}) we do not assign special semantics to the argument-generating function $\delta$, so we merge the second ($\phi$ associative) and third ($\phi$ idempotent) level for the purpose of this analysis, leaving exploration of the extra expressivity of $\delta$ to later work.} distinct \emph{levels} of asynchrony-invariant GNNs:
\begin{itemize}
    \item[\bf L1] A GNN with a \emph{commutative monoid} aggregator \cite{ong2022learnable}, $\bigoplus$, is invariant to message receiving order;
    \item[\bf L2] A GNN which, additionally, has an associative and idempotent update function \cite{tang2020towards}, $\phi$, is invariant to repeated calls to the update function.
    \item[\bf L3] A GNN which, additionally, has a message function, $\psi$, which is a \emph{monoid multimorphism}, is invariant to repeated calls to the message function.
\end{itemize}

To supplement our theory, and illustrate how progressing up these three levels practically results in a more robust algorithmic executor, we perform comparative experiments on the CLRS-30 benchmark \cite{velivckovic2022clrs}. Specifically, we closely follow the single-task CLRS-30 experimental setup of \citet{ibarz2022generalist}, incorporating inverted pointer features as discovered by \citet{bevilacqua2023neural}; we defer to these papers for concrete details on the train/test data pipeline incorporated.

In order to avoid confounding factors, we further modify the CLRS-30 baselines to focus on linear message/update functions, and do not deploy triplet messages \cite{dudzik2022graph}. Within this framework, we focus on the following three hyperparameter settings, corresponding to the three levels:
\begin{itemize}
    \item[\bf L1] Let $\bigoplus=\sum$, $\psi$ a linear layer, and $\phi$ a linear layer with ReLU activation.
    \item[\bf L2] Let $\bigoplus=\phi=\max$, and $\psi$ a linear layer.
    \item[\bf L3] Let $\bigoplus=\phi=\max$, and $\psi$ is a $\log$-semiring bilinear layer, potentially preceded by an additional linear layer (we treat this architectural choice as a per-task hyperparameter).
\end{itemize}
A $\log$-semiring bilinear layer is a matrix multiplication of the form $\psi(\vec{x}, \vec{y}) = \vec{A}_{\log}\vec{x} \times \vec{B}_{\log}\vec{y}$ where we ``re-interpret'' ``$+$'' to be the \emph{logsumexp} operator (i.e. $x +_{lse} y := \log(\exp(x) + \exp(y))$), and ``$\times$'' to be \emph{addition} (i.e. $x\times_{lse} y := x + y$). We implemented this bilinear layer using SynJax \citep{stanojevic2023synjax}.

We use the $\log$-semiring as a smooth approximation of the previously described tropical linear (where ``$+$'' would be $\max$, and ``$\times$'' would be $+$). This is a choice made due to the large sparsity in gradients for a tropical linear layer, which makes learning more challenging in practice---a phenomenon concurrently empirically observed for aggregation functions by \citet{mirjanic2023latent}.

\begin{figure}
    \centering
    \includegraphics[width=\linewidth]{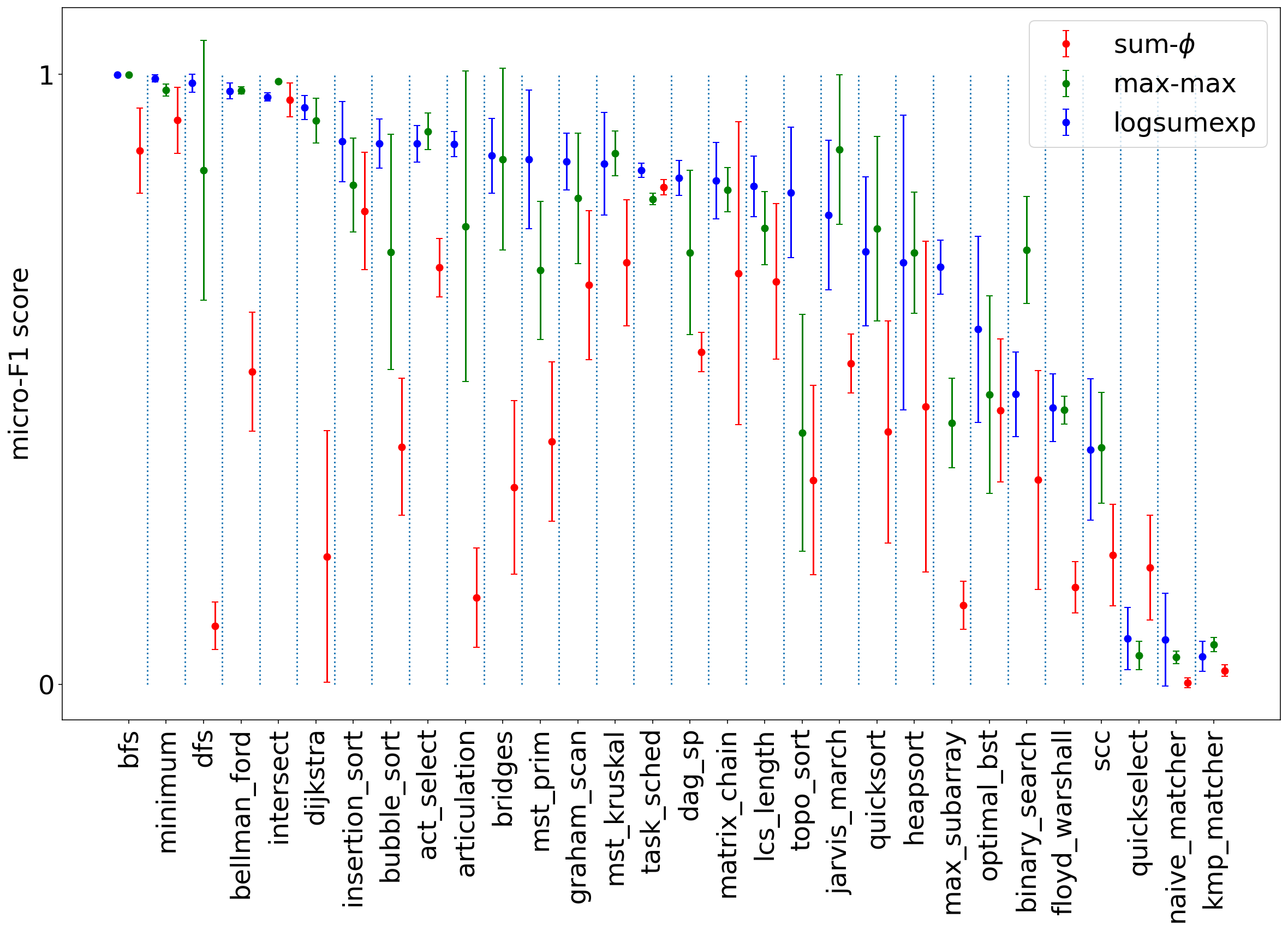}
    \caption{Test (out-of-distribution) results across all tasks in CLRS-30, for the three models described in Section \ref{sec:eval}, averaged across six seeds.}
    \label{fig:perf}
\end{figure}

The comparative results across the algorithmic execution tasks in CLRS-30, in the out-of-distribution regime, are provided in Figure \ref{fig:perf}. The overall ranking of the methods in terms of their ``performance profile'' across the thirty tasks closely matches the three levels, with the overall best results achieved by the $\log$-semiring message function, and the second-best results achieved by $\max$-$\max$ GNNs \cite{tang2020towards}. It is very useful to comment on two ``opposite'' classes of algorithms in light of these results. 

We first discuss the CLRS-30 algorithms that are ``embarrassingly parallel'', in the sense that each node can continuously broadcast its state until convergence is reached, and meaningful updates may happen anywhere in the graph. BFS \citep{moore1959shortest}, Bellman-Ford \citep{bellman1958routing} and Floyd-Warshall \citep{floyd1962algorithm} are all standard examples. For all of the above algorithms, GNNs satisfying the cocycle condition (levels 2 and 3) generalise substantially more favourably than the basic, level-1, expressive GNN.

Conversely, we have algorithms that are ``purely sequential'', in that at every step, exactly one node is broadcasting meaningful state, and exactly one of its neighbours may get a meaningful update. In this setting, nearly all messages are redundant! Therefore, having a GNN that is invariant to partial invocations of the message functions---as is the case with the level-3 $\log$-semiring architecture---may be particularly relevant, as it will reduce the variance caused by long trajectories of redundant messages. And indeed, in many representative algorithms of this class, such as DFS \citep{moore1959shortest}, articulation points, bridges and topological sorting \citep{knuth1973fundamental}, the level-3 architecture significantly reduces variance, or further improves generalisation, compared to the level-2 architecture.

We find these results, taken all together, to offer significant evidence to the utility of our theory, and we hope they can serve as an inspiration for follow-up studies and variations! In general, the level of asynchrony required by a task may not always be easily anticipated, and we foresee future work that constructs better-aligned architectures through the lens of various \emph{schedulers}.

\section{Conclusions} 

In this work we have taken a \emph{node-centric} perspective on the computation of the message passing mechanism. We note that this allows us to reason about the asynchrony of the updates on the node \emph{persistent} state, where the cocycle conditions are necessary in order to support asynchronous updates. We show that these conditions are respected for the Bellman-Ford algorithm as well as an extension of the recently-proposed PathGNN \citep{tang2020towards}, and can be used as a principle to discover other GNN formulations that allow complete asynchronous execution.

Through our analysis we provide a complementary perspective on formally describing computations carried by both algorithms and graph neural networks, which we believe to be a step forward towards further formalising the concept of algorithmic alignment that is widely relied on by approaches for neural algorithmic reasoning. Note that our approach does not compel actual GNNs to sacrifice their scalability by being asynchronous---rather, it imposes mathematical constraints on the GNNs' building blocks, such that, \emph{if we were to execute it in a particular asynchronous regime, it would behave identically to the target algorithm}. Such a correspondence, naturally, improves the level of algorithmic alignment enjoyed by the GNN---as is evident from our presented empirical results.

\bibliographystyle{unsrtnat}
\bibliography{log_2023}

\begin{thebibliography}{43}
\providecommand{\natexlab}[1]{#1}
\providecommand{\url}[1]{\texttt{#1}}
\expandafter\ifx\csname urlstyle\endcsname\relax
  \providecommand{\doi}[1]{doi: #1}\else
  \providecommand{\doi}{doi: \begingroup \urlstyle{rm}\Url}\fi

\bibitem[Gilmer et~al.(2017)Gilmer, Schoenholz, Riley, Vinyals, and
  Dahl]{gilmer2017neural}
Justin Gilmer, Samuel~S Schoenholz, Patrick~F Riley, Oriol Vinyals, and
  George~E Dahl.
\newblock Neural message passing for quantum chemistry.
\newblock In \emph{International conference on machine learning}, pages
  1263--1272. PMLR, 2017.

\bibitem[Battaglia et~al.(2016)Battaglia, Pascanu, Lai, Rezende, and
  Kavukcuoglu]{BattagliaPLRK16}
Peter~W. Battaglia, Razvan Pascanu, Matthew Lai, Danilo~Jimenez Rezende, and
  Koray Kavukcuoglu.
\newblock Interaction networks for learning about objects, relations and
  physics.
\newblock \emph{NeurIPS}, abs/1612.00222, 2016.

\bibitem[Veli{\v{c}}kovi{\'c}(2023)]{velivckovic2023everything}
Petar Veli{\v{c}}kovi{\'c}.
\newblock Everything is connected: Graph neural networks.
\newblock \emph{Current Opinion in Structural Biology}, 79:\penalty0 102538,
  2023.

\bibitem[Wang et~al.(2023)Wang, Fu, Du, Gao, Huang, Liu, Chandak, Liu,
  Van~Katwyk, Deac, et~al.]{wang2023scientific}
Hanchen Wang, Tianfan Fu, Yuanqi Du, Wenhao Gao, Kexin Huang, Ziming Liu, Payal
  Chandak, Shengchao Liu, Peter Van~Katwyk, Andreea Deac, et~al.
\newblock Scientific discovery in the age of artificial intelligence.
\newblock \emph{Nature}, 620\penalty0 (7972):\penalty0 47--60, 2023.

\bibitem[Stokes et~al.(2020)Stokes, Yang, Swanson, Jin, Cubillos-Ruiz, Donghia,
  MacNair, French, Carfrae, Bloom-Ackermann, et~al.]{stokes2020deep}
Jonathan~M Stokes, Kevin Yang, Kyle Swanson, Wengong Jin, Andres Cubillos-Ruiz,
  Nina~M Donghia, Craig~R MacNair, Shawn French, Lindsey~A Carfrae, Zohar
  Bloom-Ackermann, et~al.
\newblock A deep learning approach to antibiotic discovery.
\newblock \emph{Cell}, 180\penalty0 (4):\penalty0 688--702, 2020.

\bibitem[Liu et~al.(2023)Liu, Catacutan, Rathod, Swanson, Jin, Mohammed,
  Chiappino-Pepe, Syed, Fragis, Rachwalski, Magolan, Surette, Coombes,
  Jaakkola, Barzilay, Collins, and Stokes]{Liu2023}
Gary Liu, Denise~B. Catacutan, Khushi Rathod, Kyle Swanson, Wengong Jin,
  Jody~C. Mohammed, Anush Chiappino-Pepe, Saad~A. Syed, Meghan Fragis, Kenneth
  Rachwalski, Jakob Magolan, Michael~G. Surette, Brian~K. Coombes, Tommi
  Jaakkola, Regina Barzilay, James~J. Collins, and Jonathan~M. Stokes.
\newblock Deep learning-guided discovery of an antibiotic targeting
  acinetobacter baumannii.
\newblock \emph{Nature Chemical Biology}, May 2023.
\newblock ISSN 1552-4469.
\newblock \doi{10.1038/s41589-023-01349-8}.
\newblock URL \url{https://doi.org/10.1038/s41589-023-01349-8}.

\bibitem[Mirhoseini et~al.(2021)Mirhoseini, Goldie, Yazgan, Jiang, Songhori,
  Wang, Lee, Johnson, Pathak, Nazi, et~al.]{mirhoseini2021graph}
Azalia Mirhoseini, Anna Goldie, Mustafa Yazgan, Joe~Wenjie Jiang, Ebrahim
  Songhori, Shen Wang, Young-Joon Lee, Eric Johnson, Omkar Pathak, Azade Nazi,
  et~al.
\newblock A graph placement methodology for fast chip design.
\newblock \emph{Nature}, 594\penalty0 (7862):\penalty0 207--212, 2021.

\bibitem[Derrow-Pinion et~al.(2021)Derrow-Pinion, She, Wong, Lange, Hester,
  Perez, Nunkesser, Lee, Guo, Wiltshire, et~al.]{derrow2021eta}
Austin Derrow-Pinion, Jennifer She, David Wong, Oliver Lange, Todd Hester, Luis
  Perez, Marc Nunkesser, Seongjae Lee, Xueying Guo, Brett Wiltshire, et~al.
\newblock Eta prediction with graph neural networks in google maps.
\newblock In \emph{Proceedings of the 30th ACM International Conference on
  Information \& Knowledge Management}, pages 3767--3776, 2021.

\bibitem[DeZoort et~al.(2023)DeZoort, Battaglia, Biscarat, and
  Vlimant]{dezoort2023graph}
Gage DeZoort, Peter~W Battaglia, Catherine Biscarat, and Jean-Roch Vlimant.
\newblock Graph neural networks at the large hadron collider.
\newblock \emph{Nature Reviews Physics}, pages 1--23, 2023.

\bibitem[Davies et~al.(2021)Davies, Veli{\v{c}}kovi{\'c}, Buesing, Blackwell,
  Zheng, Toma{\v{s}}ev, Tanburn, Battaglia, Blundell, Juh{\'a}sz,
  et~al.]{davies2021advancing}
Alex Davies, Petar Veli{\v{c}}kovi{\'c}, Lars Buesing, Sam Blackwell, Daniel
  Zheng, Nenad Toma{\v{s}}ev, Richard Tanburn, Peter Battaglia, Charles
  Blundell, Andr{\'a}s Juh{\'a}sz, et~al.
\newblock {Advancing mathematics by guiding human intuition with AI}.
\newblock \emph{Nature}, 600\penalty0 (7887):\penalty0 70--74, 2021.

\bibitem[Blundell et~al.(2021)Blundell, Buesing, Davies, Veli{\v{c}}kovi{\'c},
  and Williamson]{blundell2021towards}
Charles Blundell, Lars Buesing, Alex Davies, Petar Veli{\v{c}}kovi{\'c}, and
  Geordie Williamson.
\newblock Towards combinatorial invariance for kazhdan-lusztig polynomials.
\newblock \emph{arXiv preprint arXiv:2111.15161}, 2021.

\bibitem[Williamson(2023)]{williamson2023deep}
Geordie Williamson.
\newblock Is deep learning a useful tool for the pure mathematician?
\newblock \emph{arXiv preprint arXiv:2304.12602}, 2023.

\bibitem[Veli{\v{c}}kovi{\'c} and Blundell(2021)]{velivckovic2021neural}
Petar Veli{\v{c}}kovi{\'c} and Charles Blundell.
\newblock Neural algorithmic reasoning.
\newblock \emph{Patterns}, 2\penalty0 (7):\penalty0 100273, 2021.

\bibitem[Veli{\v{c}}kovi{\'c} et~al.(2019)Veli{\v{c}}kovi{\'c}, Ying, Padovano,
  Hadsell, and Blundell]{velivckovic2019neural}
Petar Veli{\v{c}}kovi{\'c}, Rex Ying, Matilde Padovano, Raia Hadsell, and
  Charles Blundell.
\newblock Neural execution of graph algorithms.
\newblock \emph{arXiv preprint arXiv:1910.10593}, 2019.

\bibitem[Lewkowycz et~al.(2022)Lewkowycz, Andreassen, Dohan, Dyer, Michalewski,
  Ramasesh, Slone, Anil, Schlag, Gutman-Solo, et~al.]{lewkowycz2022solving}
Aitor Lewkowycz, Anders Andreassen, David Dohan, Ethan Dyer, Henryk
  Michalewski, Vinay Ramasesh, Ambrose Slone, Cem Anil, Imanol Schlag, Theo
  Gutman-Solo, et~al.
\newblock Solving quantitative reasoning problems with language models.
\newblock \emph{arXiv preprint arXiv:2206.14858}, 2022.

\bibitem[Xu et~al.(2019)Xu, Li, Zhang, Du, Kawarabayashi, and
  Jegelka]{xu2019can}
Keyulu Xu, Jingling Li, Mozhi Zhang, Simon~S Du, Ken-ichi Kawarabayashi, and
  Stefanie Jegelka.
\newblock What can neural networks reason about?
\newblock \emph{arXiv preprint arXiv:1905.13211}, 2019.

\bibitem[Dudzik and Veličković(2022)]{dudzik2022graph}
Andrew Dudzik and Petar Veličković.
\newblock Graph neural networks are dynamic programmers, 2022.

\bibitem[Faber and Wattenhofer(2022)]{faber2022asynchronous}
Lukas Faber and Roger Wattenhofer.
\newblock Asynchronous message passing: A new framework for learning in graphs.
\newblock 2022.

\bibitem[Finkelshtein et~al.(2023)Finkelshtein, Huang, Bronstein, and
  Ceylan]{finkelshtein2023cooperative}
Ben Finkelshtein, Xingyue Huang, Michael Bronstein, and {\.I}smail~{\.I}lkan
  Ceylan.
\newblock Cooperative graph neural networks.
\newblock \emph{arXiv preprint arXiv:2310.01267}, 2023.

\bibitem[Tang et~al.(2020)Tang, Huang, Gu, Lu, and Su]{tang2020towards}
Hao Tang, Zhiao Huang, Jiayuan Gu, Bao-Liang Lu, and Hao Su.
\newblock Towards scale-invariant graph-related problem solving by iterative
  homogeneous gnns.
\newblock \emph{Advances in Neural Information Processing Systems},
  33:\penalty0 15811--15822, 2020.

\bibitem[Bronstein et~al.(2021)Bronstein, Bruna, Cohen, and
  Veličković]{bronstein2021geometric}
Michael~M. Bronstein, Joan Bruna, Taco Cohen, and Petar Veličković.
\newblock Geometric deep learning: Grids, groups, graphs, geodesics, and
  gauges, 2021.

\bibitem[Battaglia et~al.(2018)Battaglia, Hamrick, Bapst, Sanchez-Gonzalez,
  Zambaldi, Malinowski, Tacchetti, Raposo, Santoro, Faulkner,
  et~al.]{battaglia2018relational}
Peter~W Battaglia, Jessica~B Hamrick, Victor Bapst, Alvaro Sanchez-Gonzalez,
  Vinicius Zambaldi, Mateusz Malinowski, Andrea Tacchetti, David Raposo, Adam
  Santoro, Ryan Faulkner, et~al.
\newblock Relational inductive biases, deep learning, and graph networks.
\newblock \emph{arXiv preprint arXiv:1806.01261}, 2018.

\bibitem[Bodnar et~al.(2023)Bodnar, Giovanni, Chamberlain, Liò, and
  Bronstein]{bodnar2023neural}
Cristian Bodnar, Francesco~Di Giovanni, Benjamin~Paul Chamberlain, Pietro Liò,
  and Michael~M. Bronstein.
\newblock Neural sheaf diffusion: A topological perspective on heterophily and
  oversmoothing in gnns, 2023.

\bibitem[Graves et~al.(2014)Graves, Wayne, and Danihelka]{graves2014neural}
Alex Graves, Greg Wayne, and Ivo Danihelka.
\newblock Neural turing machines.
\newblock \emph{arXiv preprint arXiv:1410.5401}, 2014.

\bibitem[Graves et~al.(2016)Graves, Wayne, Reynolds, Harley, Danihelka,
  Grabska-Barwi{\'n}ska, Colmenarejo, Grefenstette, Ramalho, Agapiou,
  et~al.]{graves2016hybrid}
Alex Graves, Greg Wayne, Malcolm Reynolds, Tim Harley, Ivo Danihelka, Agnieszka
  Grabska-Barwi{\'n}ska, Sergio~G{\'o}mez Colmenarejo, Edward Grefenstette,
  Tiago Ramalho, John Agapiou, et~al.
\newblock Hybrid computing using a neural network with dynamic external memory.
\newblock \emph{Nature}, 538\penalty0 (7626):\penalty0 471--476, 2016.

\bibitem[Veli{\v{c}}kovi{\'c} et~al.(2020)Veli{\v{c}}kovi{\'c}, Buesing,
  Overlan, Pascanu, Vinyals, and Blundell]{velivckovic2020pointer}
Petar Veli{\v{c}}kovi{\'c}, Lars Buesing, Matthew Overlan, Razvan Pascanu,
  Oriol Vinyals, and Charles Blundell.
\newblock Pointer graph networks.
\newblock \emph{Advances in Neural Information Processing Systems},
  33:\penalty0 2232--2244, 2020.

\bibitem[Strathmann et~al.(2021)Strathmann, Barekatain, Blundell, and
  Veli{\v{c}}kovi{\'c}]{strathmann2021persistent}
Heiko Strathmann, Mohammadamin Barekatain, Charles Blundell, and Petar
  Veli{\v{c}}kovi{\'c}.
\newblock Persistent message passing.
\newblock \emph{arXiv preprint arXiv:2103.01043}, 2021.

\bibitem[Xhonneux et~al.(2021)Xhonneux, Deac, Veli{\v{c}}kovi{\'c}, and
  Tang]{xhonneux2021transfer}
Louis-Pascal Xhonneux, Andreea-Ioana Deac, Petar Veli{\v{c}}kovi{\'c}, and Jian
  Tang.
\newblock How to transfer algorithmic reasoning knowledge to learn new
  algorithms?
\newblock \emph{Advances in Neural Information Processing Systems},
  34:\penalty0 19500--19512, 2021.

\bibitem[Ibarz et~al.(2022)Ibarz, Kurin, Papamakarios, Nikiforou, Bennani,
  Csord{\'a}s, Dudzik, Bo{\v{s}}njak, Vitvitskyi, Rubanova,
  et~al.]{ibarz2022generalist}
Borja Ibarz, Vitaly Kurin, George Papamakarios, Kyriacos Nikiforou, Mehdi
  Bennani, R{\'o}bert Csord{\'a}s, Andrew~Joseph Dudzik, Matko Bo{\v{s}}njak,
  Alex Vitvitskyi, Yulia Rubanova, et~al.
\newblock A generalist neural algorithmic learner.
\newblock In \emph{Learning on Graphs Conference}, pages 2--1. PMLR, 2022.

\bibitem[Ong and Veli{\v{c}}kovi{\'c}(2022)]{ong2022learnable}
Euan Ong and Petar Veli{\v{c}}kovi{\'c}.
\newblock Learnable commutative monoids for graph neural networks.
\newblock \emph{arXiv preprint arXiv:2212.08541}, 2022.

\bibitem[Xu et~al.(2020)Xu, Zhang, Li, Du, Kawarabayashi, and
  Jegelka]{xu2020neural}
Keyulu Xu, Mozhi Zhang, Jingling Li, Simon~S Du, Ken-ichi Kawarabayashi, and
  Stefanie Jegelka.
\newblock How neural networks extrapolate: From feedforward to graph neural
  networks.
\newblock \emph{arXiv preprint arXiv:2009.11848}, 2020.

\bibitem[Hamrick et~al.(2018)Hamrick, Allen, Bapst, Zhu, McKee, Tenenbaum, and
  Battaglia]{hamrick2018relational}
Jessica~B Hamrick, Kelsey~R Allen, Victor Bapst, Tina Zhu, Kevin~R McKee,
  Joshua~B Tenenbaum, and Peter~W Battaglia.
\newblock Relational inductive bias for physical construction in humans and
  machines.
\newblock \emph{arXiv preprint arXiv:1806.01203}, 2018.

\bibitem[Bresson and Laurent(2017)]{bresson2017residual}
Xavier Bresson and Thomas Laurent.
\newblock Residual gated graph convnets.
\newblock \emph{arXiv preprint arXiv:1711.07553}, 2017.

\bibitem[Bellman(1958)]{bellman1958routing}
Richard Bellman.
\newblock On a routing problem.
\newblock \emph{Quarterly of applied mathematics}, 16\penalty0 (1):\penalty0
  87--90, 1958.

\bibitem[Veli{\v{c}}kovi{\'c} et~al.(2022)Veli{\v{c}}kovi{\'c}, Badia, Budden,
  Pascanu, Banino, Dashevskiy, Hadsell, and Blundell]{velivckovic2022clrs}
Petar Veli{\v{c}}kovi{\'c}, Adri{\`a}~Puigdom{\`e}nech Badia, David Budden,
  Razvan Pascanu, Andrea Banino, Misha Dashevskiy, Raia Hadsell, and Charles
  Blundell.
\newblock The clrs algorithmic reasoning benchmark.
\newblock In \emph{International Conference on Machine Learning}, pages
  22084--22102. PMLR, 2022.

\bibitem[Bevilacqua et~al.(2023)Bevilacqua, Nikiforou, Ibarz, Bica, Paganini,
  Blundell, Mitrovic, and Veli{\v{c}}kovi{\'c}]{bevilacqua2023neural}
Beatrice Bevilacqua, Kyriacos Nikiforou, Borja Ibarz, Ioana Bica, Michela
  Paganini, Charles Blundell, Jovana Mitrovic, and Petar Veli{\v{c}}kovi{\'c}.
\newblock Neural algorithmic reasoning with causal regularisation.
\newblock \emph{arXiv preprint arXiv:2302.10258}, 2023.

\bibitem[Stanojevi{\'c} and Sartran(2023)]{stanojevic2023synjax}
Milo{\v{s}} Stanojevi{\'c} and Laurent Sartran.
\newblock Synjax: Structured probability distributions for jax.
\newblock \emph{arXiv preprint arXiv:2308.03291}, 2023.

\bibitem[Mirjani{\'c} et~al.(2023)Mirjani{\'c}, Pascanu, and
  Veli{\v{c}}kovi{\'c}]{mirjanic2023latent}
Vladimir~V Mirjani{\'c}, Razvan Pascanu, and Petar Veli{\v{c}}kovi{\'c}.
\newblock Latent space representations of neural algorithmic reasoners.
\newblock \emph{arXiv preprint arXiv:2307.08874}, 2023.

\bibitem[Moore(1959)]{moore1959shortest}
Edward~F Moore.
\newblock The shortest path through a maze.
\newblock In \emph{Proc. of the International Symposium on the Theory of
  Switching}, pages 285--292. Harvard University Press, 1959.

\bibitem[Floyd(1962)]{floyd1962algorithm}
Robert~W Floyd.
\newblock Algorithm 97: shortest path.
\newblock \emph{Communications of the ACM}, 5\penalty0 (6):\penalty0 345, 1962.

\bibitem[Knuth(1973)]{knuth1973fundamental}
Donald~Ervin Knuth.
\newblock Fundamental algorithms.
\newblock \emph{The art of computer programming}, 1:\penalty0 51--78, 1973.

\bibitem[Milewski(2018)]{milewski2018category}
Bartosz Milewski.
\newblock \emph{Category theory for programmers}.
\newblock Blurb, 2018.

\bibitem[Cohen and Welling(2016)]{cohen2016group}
Taco Cohen and Max Welling.
\newblock Group equivariant convolutional networks.
\newblock In \emph{International conference on machine learning}, pages
  2990--2999. PMLR, 2016.

\end{thebibliography}

\appendix
\section{Introduction to monoids}\label{app:monoids}

The central arguments of our paper rest on the construct of a \emph{monoid}. As we will unpack throughout this Appendix, monoids are an excellent abstraction for studying repeated computation in algorithms. We start with a preliminary overview of monoids, guided by the example of processing data in lists.

Consider a few common functions $f: \mathtt{list}(A) \to B$ whose input is a list type:

\begin{enumerate}
    \item $f(L) = \# L$ (the length function)
    \item $f(L) = \begin{cases*}
1 & if $\text{``hello''}\in L$ \\
0 & otherwise
\end{cases*}$, where $A=\mathtt{String}$ (the \emph{any} function)
    \item $f(L) = \prod_{a\in L} a$, where $A=\mathbb{R}$ (the product function on reals)
\end{enumerate}

In these cases (and many others), we are performing a \emph{fold}, that is, a repeated application of a binary operation over the list's elements. (possibly after a map) Specifically, for the above functions, this operation is:

\begin{enumerate}
    \item The operation $+$ on the set $\mathbb{N}$ of natural numbers.
    \item The operation $\mathtt{or}$ on the set $\{0, 1\}$ of Boolean values.
    \item The operation $\times$ on the set $\mathbb{R}$ of real numbers.
\end{enumerate}

The fact that these functions arise from such an operation has deep consequences for the parallelisation of the computation, as well as its semantics. Further, the binary operations considered above all have properties of interest. In all three cases, we may arbitrarily split the list into sublists, perform the computation in parallel on each sublist, and then compute the aggregate result---and this will not affect the final result of the fold operation.

With this is mind, we define a \emph{monoid} to be a triple $(M, \cdot, 1)$ consisting of a set $M$, a binary operation $\cdot : M\times M \to M$, and a ``neutral'' element $1\in M$, satisfying the following axioms for all $a,b,c,x\in M$:
\begin{equation}\label{monoid-axioms}
\begin{split}
1\cdot x = x \cdot 1 = x \\
a\cdot(b\cdot c) = (a\cdot b)\cdot c
\end{split}
\end{equation}
The reader may have encountered these axioms when studying \emph{groups}, particularly in the context of geometric deep learning \citep{bronstein2021geometric}. Groups have an additional axiom of \emph{invertibility}; therein, every element $x\in M$ must have an inverse element $x^{-1}\in M$ such that $x \cdot x^{-1} = x^{-1} \cdot x = 1$. That being said, we note that none of the three binary operations considered above are invertible, hence none of these structures are groups. Since programming is usually focused on non-invertible operations, groups arise quite rarely in algorithmic computation, while monoids are extremely common.

We also note a commonality with \emph{category theory}: a monoid is exactly the same thing as a category with only one object. We will not explore this perspective further here, but the reader can consult \citet{milewski2018category} for more information about the categorical approach to monoids.

When defining a mathematical structure, we should also say how to define \emph{arrows}, or structure-preserving maps, between them. A \emph{monoid homomorphism} $f: (M,\cdot_M, 1_M) \to (M', \cdot_{M'}, 1_{M'})$ is defined to be a function $f:M\to M'$ satisfying the following axioms.
\begin{equation}\label{homomorphism-axioms}
\begin{split}
f(1_M) &= 1_{M'} \\
f(a\cdot_M b) &= f(a)\cdot_{M'} f(b)
\end{split}
\end{equation}
Such a function preserves the structure contained in $M$ when mapping its elements into $M'$. If a homomorphism is bijective then its inverse is automatically a homomorphism; we call such a homomorphism an \emph{isomorphism}.

Just as with groups, we are typically not only interested in monoids themselves, but how they \emph{act} on data. A set $S$ is said to be equipped with a \emph{left $M$-action} when we have a binary operation $\bigcdot: M\times S \to S$ such that the following axioms hold for all $m,n\in M$ and $s\in S$:
\begin{equation}\label{action-axioms}
\begin{split}
1 \bigcdot s &= s \\
n\bigcdot (m\bigcdot s) &= (n\cdot m)\bigcdot s
\end{split}
\end{equation}
These axioms are equivalent to saying that we have a map $\rho:M\to [S,S]$ from $M$ to endofunctions on $S$, satisfying $\rho(1) = \mathrm{id}_S$ and $\rho(n\cdot m) = \rho(n)\circ \rho(m)$, where $\mathrm{id}_S$ is the \emph{identity function} on $S$\footnote{Defined as $\mathrm{id}_S(s)=s$ for all $s\in S$.}, and $\circ$ is function composition. The map $\rho$ is sometimes also referred to as a \emph{representation} of the monoid.

We will also make use of the notion of \emph{right $M$-action}, which is the same as the above but with all orders of operation reversed. Whenever not specified, actions are assumed to be on the left.

All monoids act on themselves. Specifically, if we define $m\bigcdot s := m\cdot s$ for $m,s\in M$, we can see by comparing Equations \ref{monoid-axioms} and \ref{action-axioms} that $\bigcdot$ satisfies the axioms of a left action. This is called the \emph{left regular representation} of $M$. Analogously, we can define a \emph{right regular representation}. This gives a large class of examples of monoid actions.


\subsection{Monoid actions as state machines}

Now, we justify why monoids are excellent for representing repeated computation, by relating monoid actions to the foundational computer science concept of \emph{state machines}\footnote{Unlike the more standard computer science construct, the state machines we consider here may be infinite. Constraining the states and transitions in an appropriate way may allow us to recover standard hierarchies of computability, such as the Chomsky hierarchy.}.

Recall that we use the notation $[A, B]$ to denote the set of all possible functions mapping $A$ to $B$.

Define a \emph{state machine} to be a triple $(S,\Sigma, \tau)$, consisting of a set $S$ of \emph{states}, a set $\Sigma$ of \emph{symbols}, and a \emph{transition function} $\tau: \Sigma \to [S,S]$ that interprets each symbol as a transition of states. In many cases, state machines also have distinguished starting and accepting states, but these are not important for the present discussion.

Given a monoid $M$ and a set $S$ equipped with an $M$-action $\rho$, we can quickly conclude that $(S, M, \rho)$ fits the definition of a state machine. However, it turns out that we can go in the reverse direction as well: every state machine can be seen as specifying an action of a monoid.

To see this, suppose that $(S,\Sigma, \tau)$ is a state machine, and let $\Sigma^\ast$ denote the collection of words drawn from $\Sigma$. That is, $\Sigma^\ast$ comprises strings of the form $\sigma = \sigma_1 \cdots \sigma_n$ (where $\sigma_i\in\Sigma$), including the empty string, $\emptyset$. Note that, for each such word $\sigma\in\Sigma^\ast$, we have an associated composed transformation $\tau_\sigma := \tau(\sigma_1)\circ \cdots \circ \tau(\sigma_n)$. This function $\tau_\sigma\in [S,S]$ describes the final state reached after consuming the sequence of symbols in $\sigma$, for a given initial state.

Then, the set of all composed transitions, $M=\{\tau_\sigma \mid \sigma\in\Sigma^\ast \}$ forms a monoid $(M, \circ, \mathrm{id}_S)$ under the function composition operation $\circ$, with the identity transformation as the neutral element. Further, $S$ can be equipped with an $M$-action, defined by $\tau_\sigma \bigcdot s := \tau_\sigma(s)$, i.e. $\rho = \tau$.

Note that some information contained within the state machine is lost when converting it to a monoid in this way. To recover the original state machine, we must specify the function $\Sigma \to M$ assigning each symbol to its transformation. However, since this amounts to labeling a particular subset of transitions\footnote{Specifically, the subset being labelled needs to be a \emph{generating set}, meaning that every element of the monoid can be obtained from some composition of labelled elements.}, it does not affect the structure of the computations performed by the state machine.

But information is also lost when passing from monoid actions to state machines, since we have neglected the monoid structure of $M$ by discarding its binary operation and neutral element. The only structure that remains is contained in the monoid action, $\bigcdot$, but there are no guarantees that it will fully disambiguate different elements of the monoid.

To formalise this, we say that a monoid action, $\bigcdot : M\times S\to S$, is \emph{faithful} if, for all $m, n\in M$, if $m$ and $n$ are different, they must act differently on some state in $S$ using $\bigcdot$. Or, equivalently,
\begin{equation}
    (\forall s\in S.\ m\bigcdot s = n\bigcdot s) \Longleftrightarrow m = n
\end{equation}
If a monoid action is not faithful, then we will have at least two \emph{different} monoid elements mapping to exactly the same transitions in the state machine, and hence they will be indistinguishable. In general, state machines correspond to only the faithful actions. As we will see in Example \ref{number-wheels}, not all monoid actions are faithful.

\begin{example}\label{open-closed}
A standard example of a state machine consists of two states, $S=\{\operatorname{door\ open}, \operatorname{door\ closed}\}$, with $\Sigma = \{\operatorname{open}, \operatorname{close}\}$ performing the obvious transitions, with the convention that $\operatorname{open}$ does nothing to the open door, and $\operatorname{close}$ does nothing to the closed door.

\begin{center}
\begin{tikzpicture}
\node[state, align=center, minimum width=5em] (q_0) {door\\ open};
    \node[state, align=center, minimum width=5em,right=5em of q_0] (q_1) {door\\ closed};
    \draw[-stealth] (q_0) edge[bend left] node[above] {close} (q_1);
    \draw[-stealth] (q_1) edge[bend left] node[below] {open} (q_0);
\end{tikzpicture}
\end{center}

To see how to interpret this as a monoid action, we note the following composition laws:
\begin{equation}
\begin{split}
    \operatorname{open} \circ \operatorname{open} = \operatorname{open} \\
    \operatorname{open} \circ \operatorname{close} = \operatorname{open} \\
    \operatorname{close} \circ \operatorname{open} = \operatorname{close} \\
    \operatorname{close} \circ \operatorname{close} = \operatorname{close} \\
\end{split}
\end{equation}
Such a structure already exhibits the required associativity properties. In order to obtain a monoid from it, note that we are missing a neutral transformation. One simple way to achieve this is to augment $\Sigma$ with a single symbol, the identity $\operatorname{no-op}$. So we have $M = \{\operatorname{no-op}, \operatorname{open}, \operatorname{close}\}$. We can visualise the action of $M$ on $S$ by showing how these operations act on both states:

\begin{center}
\begin{tikzpicture}
\node[state, align=center, minimum width=5em] (q_0) {door\\ open};
    \node[state, align=center, minimum width=5em,right=5em of q_0] (q_1) {door\\ closed};
\draw[-stealth] (q_0) edge[bend left] node[above] {close} (q_1);
  \draw[-stealth] (q_1) edge[bend left] node[below] {open} (q_0);


  
  \draw[-stealth] (q_0) edge[in=136.3,out=106.3,loop, >=stealth] node[above] {open} (q_0);
  \draw[-stealth] (q_0) edge[in=253.7,out=223.7,loop, >=stealth] node[below] {no-op} (q_0);
  \draw[-stealth] (q_1) edge[in=73.7,out=43.7,loop, >=stealth] node[above] {close} (q_1);
  \draw[-stealth] (q_1) edge[in=316.3,out=286.3,loop, >=stealth] node[below] {no-op} (q_1);

\end{tikzpicture}
\end{center}

\end{example}

\begin{example}\label{number-wheels}

Let $W_n$ be the set $\{0,1,\ldots , n-1\}$, considered modulo $n$. We can make $W_n$ into a (modular) counting state machine with $\Sigma = \{+1\}$. Here is a picture when $n=3$:

\begin{center}
\begin{tikzpicture}

\foreach \a in {0,1,2}{
\draw (\a*360/3: 2cm) node[state,align=center] (q_\a) {$\a$};
}
  \draw[-stealth] (q_0) edge[bend right=45] node[above right] {$+1$} (q_1);
  \draw[-stealth] (q_1) edge[bend right=45] node[above left] {$+1$} (q_2);
  \draw[-stealth] (q_2) edge[bend right=45] node[below] {$+1$} (q_0);

\end{tikzpicture}
\end{center}

Note that this state machine has a cyclic symmetry, so for example we get the same state machine if we relabel $0\mapsto 1$, $1\mapsto 2$, $2\mapsto 0$.  For this reason, we do not identify $W_n$ with the set $\mathbb{Z}/n$, as the latter has an unambiguous additive unit element, $0$.

As above, in order to get a monoid we need to augment with any remaining transitions associated with sequences of symbols: the no-op ($+0$) corresponding to an empty string, as well as $+2 := (+1)\circ (+1)$ and more generally $+k := (+1)^{\circ k}$. The corresponding state machine for $n=3$ now looks like this:

\begin{center}
\begin{tikzpicture}

\foreach \a in {0,1,2}{
\draw (\a*360/3: 2cm) node[state,align=center] (q_\a) {$\a$};
}
  \draw[-stealth] (q_0) edge[bend right=45] node[above right] {$+1$} (q_1);
  \draw[-stealth] (q_1) edge[bend right=45] node[above left] {$+1$} (q_2);
  \draw[-stealth] (q_2) edge[bend right=45] node[below] {$+1$} (q_0);
  \draw[-stealth] (q_1) edge node[below left] {$+2$} (q_0);
  \draw[-stealth] (q_0) edge node[above left] {$+2$} (q_2);
  \draw[-stealth] (q_2) edge node[right] {$+2$} (q_1);

  \draw[-stealth] (q_0) edge[loop right, >=stealth] node[right] {$+0$} (q_0);
  \draw[-stealth] (q_1) edge[in=135,out=105,loop, >=stealth] node[above] {$+0$} (q_1);
  \draw[-stealth] (q_2) edge[in=255,out=225,loop, >=stealth] node[below] {$+0$} (q_2);

\end{tikzpicture}
\end{center}

But we note that, because of the modular nature of the counting done by the state machine, we have $+n = +0$. This implies that, from the perspective of the state machine, all transitions $+k$ with $k\geq n$, are already described by the shorter ones. Accordingly, the monoid whose action corresponds to this state machine is the \emph{cyclic group} of order $n$, denoted by $(\mathbb{Z}/n, +, 0)$.

However, it can be useful to think of $W_n$ as having an action of the monoid spanning the entire set of natural numbers $(\mathbb{N}, +, 0)$\footnote{Note that $(\mathbb{N}, +, 0)$ is not a group, as there are no inverses.}, rather than just $(\mathbb{Z}/n, +, 0)$. If we define $k\bigcdot s := (+k)(s)$, we see that the axioms for an action are satisfied, even if $k\geq n$. But since, e.g., $0\bigcdot s = n\bigcdot s$ for all $s\in W_n$, this is no longer a faithful action, and hence it does not arise from a state machine. We will discuss how to repair this defect in the following section.

\end{example}

\subsection{Extension problems and cocycles}

If $f: M\to M'$ is a morphism of monoids, we define the \emph{kernel}, $\ker{f} := f^{-1}(1_{M'})$, as the set of all elements of $M$ that map to $1_{M'}$. For general monoids, the kernel is not very informative. But when $M$ and $M'$ are groups, the kernel gives us quite a bit of information about $f$---for example, $f$ is injective if and only if $\ker{f}=\{1_M\}$.

A classic question in group theory, and a motivating example for the development of group cohomology, is this: Given two groups $G$ and $A$, how can we describe all the groups $G'$ with a surjective homomorphism $f:G'\to G$, such that $\ker{f}$ is isomorphic to $A$? In other words, we seek to describe all groups $G'$ that can be mapped to $G$ with a function $f : G'\to G$, in such a way that:
\begin{itemize}
    \item All elements of $G$ are mapped by some element in $G'$, i.e., $\forall x\in G.\exists x'\in G'. f(x') = x$, and
    \item The structure of all elements of $G'$ that map into the neutral element for $G$ ($1_G$) is isomorphic to the structure of $A$.
\end{itemize}
Such groups $G'$ are called \emph{extensions of $G$ by $A$}. For example, if $G=\mathbb{Z}/2$ and $A=\mathbb{Z}/3$, then one can easily define a ``trivial'' extension $\mathbb{Z}/6 \cong \mathbb{Z}/2\times \mathbb{Z}/3$. This extension is valid because we can map elements of $\mathbb{Z}/6$ to all elements in $G$ by using the function $x\mapsto x\pmod{2}$. Then, all elements of $\mathbb{Z}/6$ mapping to $0$ would be $\{0, 2, 4\}$, which has exactly the same structure as $A$.

Note that there also exists a nontrivial extension of $\mathbb{Z}/2$ by $\mathbb{Z}/3$: the symmetric group $S_3$, whose elements are permutations of three elements. Here the surjection $S_3 \to \mathbb{Z}/2$ is given by taking a permutation to its \emph{sign}.

The analogous question of extensions for two monoids is much harder than for groups, due to severe technical challenges in defining cohomology, and we will not attempt to analyse it in detail here. But we will show how certain functions---which will correspond to the notion of ``cocycle'' we used in the body of our paper---may be used to construct certain extensions of monoids.

Within the state machine $W_n$ given by our previous example, we were unable to distinguish between the instructions $+0$ and $+n$. We can recall from our early education why this is: in the transition from state $n-1$ to state $0$, we should produce a \emph{carry}, which would itself be applied to advance a different state machine. If this next machine also has the structure of $W_n$, this process can be iterated, giving us arithmetic in base $n$.

Focusing on just the ones digit, we can see that carrying is really describing $\mathbb{N}$ as an extension of $G=\mathbb{Z}/n$, the monoid associated to the state machine, by $A=n\mathbb{N}\cong \mathbb{N}$, the monoid for carries. That is, $n\mathbb{N}$ is the kernel of the natural map $\mathbb{N}\to \mathbb{Z}/n$.

Let's look in more detail at how this works, in terms of the state machine $S=W_n$ acted on by $M=\mathbb{N}$, and producing carries in $A=\mathbb{N}$. We describe carrying as a function $\delta: M\times S \to A$. Given an instruction $m$ acting on a state $s$, $\delta_m(s)$ describes the carry produced during the transition $s\mapsto m\cdot s$.

For $\delta$ to be consistent with the equations for a monoid action, we need two things to be true. First, the identity transition $s\mapsto 1_M \cdot s = s$ must produce an identity carry. And second, given two instructions $m$ and $n$, we must get the same total carry from the transition $s\mapsto (nm)\cdot s$ and the pair of transitions $s\mapsto m\cdot s \mapsto n\cdot (m\cdot s)$. We can summarise this in the following pair of equations, which we'll call the \emph{carry equations}:
\begin{equation}\label{eqn:carryeq}
\begin{split}
\delta_{1}(s) &= 0 \\
\delta_{n\cdot m}(s) &=  \delta_n(m\bigcdot s) + \delta_m(s)
\end{split}
\end{equation}
Here we have written $(M,\cdot, 1)$ in multiplicative notation and $(A, +, 0)$ in additive notation, to avoid confusion between the two operations.

In the carry example, satisfying Equation \ref{eqn:carryeq} guarantees we can process and emit carries for multiple numbers to be added in any order. For example, if we want to make two actions, $(+2), (+3)$ on a state $s\in W_n$, we could either emit a carry for $s + 2$ and then combine it with another carry emitted for $(s + 2) + 3$, or only emit a single carry for $s + 5$---the final carry is guaranteed to be the same. 

Now it is worthwhile to recall that, in the context of graph neural networks, this is exactly the abstraction we used to reason about the interplay between the incoming messages in a node ($M$) and the emitted node arguments ($A$) in response to those messages. In the GNN example, should its update function, $\phi$, satisfy a condition as in Equation \ref{eqn:carryeq}, when updating a node's state based on incoming messages $\vec{n}$ and $\vec{m}$, we can either first aggregate them and emit one update ($\phi(\vec{x}, \vec{n}\bigoplus\vec{m})$), or update based on the first one, then update based on the second one ($\phi(\phi(\vec{x}, \vec{n}), \vec{m})$).

To relate the carry equation back to extensions, we can verify the following by direct calculation:

\begin{proposition}
The rule $m\bigcdot (s, a) := (m\bigcdot s, \delta_m(s) + a)$ is a monoid action of $M$ on $S\times A$ if and only if $\delta$ satisfies the carry equations.
\end{proposition}

To bring this in line with more standard mathematical machinery, we note that the set of ``readout functions'' $[S, A]$ inherits an action of $M$ from $S$: $(f\bigcdot m)(s) := f(m\bigcdot s)$. Note that this is a \emph{right} action, while the action on $S$ is a \emph{left} action.

Writing $\delta$ in its curried form $D: M\to [S, A]$, we can rewrite the carry equation as follows:
\begin{equation}\label{cocycle-appendix}
\begin{split}
D(1) &= 0 \\
D(n\cdot m) &= D(n)\bigcdot m + D(m)
\end{split}
\end{equation}
Functions $D$ with these properties appear in the literature under three names: \emph{crossed homomorphisms} (to emphasise that this specialises to Equation \ref{homomorphism-axioms} when $M$ acts trivially), \emph{derivations} (to emphasise that this generalises to the Leibniz equation when $M$ acts nontrivially on the left as well as the right), and \emph{$1$-cocycles} (to emphasise the possibly higher-dimensional situation of $k$-cocycles $M^k \to [S, A]$).

Note that it is more common to describe these objects in terms of left actions rather than right actions. But the two viewpoints are compatible, which we discuss in more detail in Appendix \ref{app:stars}.

\section{Star-monoids and group cocycles}\label{app:stars}

We have related left actions of a monoid $M$ to state machines. The reader may ask: does the right action on $[S,A]$ we describe above also relate to state machines?

Define $M^{op}$, the \emph{opposite} of $M$, to be the monoid with the same underlying set as $M$, but reversed multiplication. Then it is easy to verify that a left action of $M$ is equivalent to a right action of $M^{op}$, and likewise a left action of $M^{op}$ is equivalent to a right action of $M$.

So $[S,A]$ can be interpreted as a state machine acted on by $M^{op}$. But with some additional structure, we will show that we can reinterpret this as a state machine acted on by $M$.

Define a \emph{star-monoid} to be a monoid $M$ equipped with an isomorphism $\ast: M\to M^{op}$ between $M$ and its opposite. A homomorphism of star-monoids $f:M\to N$ is a homomorphism of the underlying monoids, satisfying $f(m^\ast) = f(m)^\ast$.

There is also a short categorical definition of star-monoids. Just as monoids can be identified with one-object categories, star-monoids can be identified with one-object \emph{dagger categories}, a common tool in category theory.  For example, the category of real vector spaces with a fixed basis has a dagger structure given by the transpose.

The category of commutative monoids has an embedding into the category of star-monoids: we simply let $\ast$ be the identity function. By commutativity, $\ast$ is an isomorphism $M\to M^{op}$.

But groups are also star-monoids, in a completely different way. In any group $G$, we have the equation $(gh)^{-1} = h^{-1}g^{-1}$, so $g^\ast = g^{-1}$ gives an isomorphism $G\to G^{op}$. So the category of groups embeds into the category of star monoids (beware: abelian groups can be treated as either commutative monoids or groups, and the two approaches give incompatible stars except in special cases).

This means that, if $M$ has the structure of a star-monoid, we can rewrite Equation \ref{cocycle-appendix} as $D(n\cdot m) =  m^\ast \bigcdot D(n) + D(m)$. So, if $G$ and $A$ are groups and $A$ is abelian, we can take $M=G^{op}$ and our definition coincides with the standard definition of a group 1-cocycle for $G$ with coefficients in $[S,A]$.\footnote{We can even get non-abelian cocycles by replacing $A$ with $A^{op}$ as well.}

In fact, it is a completely standard technique to define the left action of a group $G$ on functions using the opposite group, i.e. $(gf)(s) := f(g^\ast s) = f(g^{-1}s)$, a definition which is heavily leveraged in group convolutional neural networks \citep{cohen2016group}.

Since this technique of converting from right actions to left actions produces exactly the correct definitions for groups, we propose that star-monoids, and even more generally, dagger-categories, are an appropriate setting for generalising equivariant convolutions to the setting of monoids and categories---but we leave this to future work.


\end{document}